\documentclass{article}
\pdfoutput=1
\usepackage{microtype}
\usepackage{graphicx}
\usepackage{subfigure}
\usepackage{format}
\usepackage{booktabs}
\usepackage{amsthm}
\newtheorem{innercustomthm}{Theorem}
\newenvironment{cthm}[1]
  {\renewcommand\theinnercustomthm{#1}\innercustomthm}
  {\endinnercustomthm}
\newtheorem{innercustomlemma}{Lemma}
\newenvironment{clemma}[1]
  {\renewcommand\theinnercustomlemma{#1}\innercustomlemma}
  {\endinnercustomlemma}
  
 \newtheorem{innercustomproposition}{Proposition}
\newenvironment{cprop}[1]
  {\renewcommand\theinnercustomproposition{#1}\innercustomproposition}
  {\endinnercustomproposition}
\DeclareMathOperator{\err}{err}

\DeclareMathOperator{\Enta}{Ent_{Tr_1}}
\DeclareMathOperator{\Entb}{Ent_{Tr_2}}
\usepackage{epstopdf}
\usepackage[accepted]{icml2020}
\icmltitlerunning{On the Consistency of Top-$k$ Surrogate Losses}

\usepackage{comment}
\begin{document}

\twocolumn[
\icmltitle{On the Consistency of Top-\texorpdfstring{$k$}{} Surrogate Losses}

\icmlsetsymbol{equal}{*}

\begin{icmlauthorlist}
\icmlauthor{Forest Yang\texorpdfstring{\footnotemark}{}}{}
\icmlauthor{Sanmi Koyejo\texorpdfstring{\footnotemark}{}}{}
\end{icmlauthorlist}

\icmlcorrespondingauthor{Forest Yang}{forestyang@berkeley.edu}
\icmlcorrespondingauthor{Sanmi Koyejo}{sanmi@illinois.edu}

\vskip 0.3in
]

\begin{abstract}
The top-$k$ error is often employed to evaluate performance for challenging classification tasks in computer vision as it is designed to compensate for ambiguity in ground truth labels. This practical success motivates our theoretical analysis of consistent top-$k$ classification. 
Surprisingly, it is not rigorously understood when taking the $k$-argmax of a vector is guaranteed to return the $k$-argmax of another vector, though doing so is crucial to describe Bayes optimality; we do both tasks.
Then, we define top-$k$ calibration and show it is necessary and sufficient for consistency. Based on the top-$k$ calibration analysis, we propose a class of top-$k$ calibrated Bregman divergence surrogates. Our analysis continues by showing previously proposed hinge-like top-$k$ surrogate losses are not top-$k$ calibrated and suggests no convex hinge loss is top-$k$ calibrated. On the other hand, we propose a new hinge loss which is consistent. We explore further, showing our hinge loss remains consistent under a restriction to linear functions, while cross entropy does not. Finally, we exhibit a differentiable, convex loss function which is top-$k$ calibrated for specific $k$. 
\end{abstract}

\section{Introduction}
\footnotetext[1]{University of California, Berkeley, work completed while an intern at Google Research Accra. Email: forestyang@berkeley.edu}
\footnotetext[2]{Google Research Accra \& University of Illinois at Urbana Champaign. Email: sanmi.koyejo@gmail.com}
\label{intro}
Consider a multiclass classifier which is granted $k$ guesses, so its prediction is declared error-free only if any one of the guesses is correct. This conceptually defines the top-$k$ error~\citep{Akata2012TowardsGP}. Top-$k$ error\footnote{The top-$k$ error is simply 1 - top-$k$ accuracy, thus the metrics are equivalent.} is popular in computer vision, natural language processing, and other applied problems where there are a large number of possible classes, along with potential ambiguity regarding the label of a sample and/or when a sample may correspond to multiple labels, e.g., when an image of a park containing a pond may be correctly labeled as either a park or a pond~\citet{ILSVRC15, Xiao2010SUNDL, Zhou2018PlacesA1}.

Like the zero-one loss for binary classification, the top-$k$ error is computationally hard to minimize directly because it is discontinuous and only has zero gradients. Instead, practical algorithms depend on minimizing a surrogate loss, often a convex upper bound~\citep{Lapin2015TopkMS, Lapin2016LossFF}. To this end, the corresponding predictive model is most often trained to output a continuous-valued score vector, and the classes corresponding to the top $k$ entries of the score vector constitute the classification prediction~\citep{Lapin2018AnalysisAO}.
While popular in practice, there is limited work on the theoretical properties of top-$k$ error and its surrogate losses. We are particularly interested in the \textit{consistency} of surrogate losses, which states whether the learned classifier converges to the population optimal prediction (commonly known as the Bayes optimal) in the infinite sample limit.

\paragraph{Main Contributions.}
Our contributions are primarily theoretical, and are outlined as follows:
\begin{itemize}
    \item We characterize Bayes-optimal scorers for the weighted top-$k$ error, i.e., a slight generalization top-$k$ error with class-specific weights. The scorers are functions which predict continuous vectors, so the $k$ maximum arguments define the prediction. Our analysis highlights the top-$k$ preserving property as fundamental to top-$k$ consistency, then outlines the notion of calibration which is necessary and sufficient to construct consistent top-$k$ surrogate losses.
    \item We propose a family of consistent (weighted) top-$k$ surrogate losses based on Bregman divergences. We show the inconsistency of previously proposed top-$k$ hinge-like surrogate losses and propose new ones, one of which is (weighted) top-$k$ consistent. Since any convex hinge loss must have form similar to the ones proved inconsistent, this suggests that consistent hinge losses must be nonconvex.
    \item We further prove the consistency of the new hinge loss when given top-$k$ separable data and restricted to linear predictors. On the other hand, we also show that cross entropy, while being top-$k$ consistent in the unrestricted setting, is not consistent when restricted to linear models. 
    \item A loss being convex and differentiable can often lead to strong guarantees. Investigating this, we find that while a convex and differentiable top-$k$ calibrated loss function must also be calibrated for $k'\leq k$, but by exhibiting a counterexample we show that it need not be calibrated for $k' > k$.
    \item We employ these losses in synthetic experiments, observing aspects of their behavior which reflect our theoretical analysis.
\end{itemize}
Taken together, our results contribute to the fundamental understanding of top-$k$ error and its (in)consistent surrogates.

\subsection{Notation}
For any $N\in\Z^+$, we use the notation $[N]=\{1,\ldots, N\}$.
We assume there are $M$ classes and denote the input space as $\cX$. 
We also denote the $i$th coordinate basis vector as $e_i$; the dimension should be clear from context. $\cY=[M]$ is the discrete label space.
The data is assumed to be generated i.i.d. from some distribution $\mathbb{P}$ over $\cX\times\cY$.

Define the probability simplex
$\Delta_M:=\{v\in\R^M\mid\forall m\in[M],\, v_m\geq 0,\, \sum_{m=1}^M v_m = 1\}$,
and let $\eta(x)\in\Delta_M$
be the conditional distribution of $y\in\cY$ given $x\in\cX$, i.e. $\eta(x)_m 
= P(y=m\mid X=x)$.
Furthermore, given a vector $v\in\mathbb{R}^m$, let $v_{[j]}$ denote the $j$th greatest entry of $v$.
For example, if $v=(1,4,4,2)$, then $v_{[1]} = 4, v_{[2]} = 4, v_{[3]}= 2, v_{[4]} = 1$.

\subsection{Related Work}
The statistical properties of surrogates for binary classification are well-studied \citep{Zhang2004StatisticalBA, Bartlett2003ConvexityC}. Furthermore, many of these results have been extended to multiclass classification with the accuracy metric \citep{Zhang2004StatisticalAO, Tewari2005OnTC}. Usually, $y\in\{1,\ldots,M\}$, $s\in\mathbb{R}^M$ is a vector-valued score, and the prediction is the index of the entry of $s$ with the highest value. There have also been recent studies on a general framework for consistent classification with more general concave and fractional linear multiclass metrics~\citep{narasimhan2015consistent}. In the realm of multilabel classification, there is work on extending multiclass algorithms to multilabel classification \citep{Lapin2018AnalysisAO}, characterizing consistency for multilabel classification \citep{Gao:2013:CML:3015367.3015369}, and constructing a general framework for consistent classification with multilabel metrics~\citep{Koyejo2015ConsistentMC}.

On the other hand, statistical properties such as consistency of surrogate loss functions for the top-$k$ error are not so thoroughly characterized. It is known that softmax loss $-\log\left(\frac{e^{s_y}}{\sum_{m=1}^M e^{s_m}}\right)$ is top-$k$ consistent and that the multiclass hinge loss $\max_{m\in[M]}\{\1[m\neq y] + s_m - s_y\}$ proposed by \citet{Crammer2001OnTA} is top-$k$ inconsistent \citep{Zhang2004StatisticalAO}. However, the consistency of recently proposed improved top-$k$ surrogates such as proposals in \citet{Berrada2018SmoothLF, Lapin2015TopkMS, Lapin2016LossFF, Lapin2018AnalysisAO} has so far remained unresolved. Our work resolves some of these open questions by showing their inconsistency, in addition to providing a more robust framework for top-$k$ consistency.
\section{Top-$k$ consistency}
We begin by formally defining the top-$k$ error.
\begin{definition}[Top-$k$ error]
\label{topk}
Given label vector $y\in\mathcal{Y}$ with $y_l=1$ and prediction $s\in\mathbb{R}^M$, the top-$k$ error is defined as 
\begin{equation}
    \err_k(s, y) = \1[l\not\in r_k(s)],
\end{equation}
where $r_k: \R^M\to \{J:J\subset[M],\,|J|=k\}$ is a top-$k$ selector which selects the $k$ indices of the greatest entries of the input, breaking ties arbitrarily. Different $r_k$'s correspond to different ways of breaking ties; we will take a worst-case perspective for ensuring Bayes optimality.
\end{definition}
In general, $s$ is the output of some predictor $\theta$ given a sample $x\in\cX$. The goal of a classification algorithm under the top-$k$ metric is to learn a predictor $\theta:\cX\to\R^M$ that minimizes the risk
\begin{equation*}
L_{\err_k}(\theta) := \E_{(x,y)\sim\P}[\err_k(\theta(x), y)].
\end{equation*}
Given $s\in\R^M$ and $\eta\in\Delta_M$, we may define the conditional risk
\begin{equation*}
L_{\err_k}(s, \eta) := \E_{y\sim\eta}[\err_k(s, y)]. 
\end{equation*}
Furthermore, we define optimal risk and conditional risk
\begin{align*}
    L_{\err_k}^* &:= \inf_{\theta:\cX\to\R^M} L_{\err_k}(\theta), \\
    L_{\err_k}^*(\eta) &:= \inf_{s\in\R^M} L_{\err_k}(s, \eta).
\end{align*}
Analogous population statistics for arbitrary loss functions $\psi: \R^M\times \cY\to \R$ are denoted by swapping the metrics, e.g. $\psi$ risk is defined as $L_{\psi}(\theta) := \E_{(x,y)\sim\P}[\psi(\theta(x), y)]$.

\subsection{Bayes Optimality}
Here we define and characterize Bayes optimal predictors for the top-$k$ error. 
\begin{definition}[Top-$k$ Bayes optimal]
\label{definition:topk}
The predictor $\theta^*: \cX\to\R^M$ is top-$k$ Bayes optimal if 
\begin{equation*}
    L_{\err_k}(\theta^*) =  L_{\err_k}^*.
\end{equation*}
\end{definition}


We remark that it is much less obvious which $s$, given $\eta$, are optimal for (minimize) the top-$k$ conditional risk $L_{\err_k}(s, \eta)$ than for the binary conditional risk, where $s\in\R$ is optimal (for a worst case selector) iff $\eta>1/2\implies s>0$ and $\eta<1/2 \implies s < 0$. This has led to seemingly natural but incorrect statements in prior work. For example, \citet{Lapin2016LossFF, Lapin2018AnalysisAO} write
\begin{gather*}
s \in \argmin_s L_{\err_k}(s, \eta) \iff \\
\{y\mid s_y \geq s_{[k]}\}\subseteq \{y\mid \eta_y \geq \eta_{[k]}\},
\end{gather*}
which says that the top-$k$ indices of $s$ are contained in the top-$k$ indices of $\eta$. However,
consider the following counter-example. Let $s=(0, 1, 1)$, $\eta=(1, 0, 0)$ and $k=2$. Note $s_{[k]} = 1, \eta_{[k]} = 0$. Then, $\{y\mid s_y \geq s_{[k]}\}=\{2,3\}\subseteq \{y\mid \eta_y \geq \eta_{[k]}\}=\{1,2,3\}$. By the above definition, $s$ is considered optimal. Yet, it is not, because for any top 2-selector $r_2(s)=\{2,3\}$, which has $100\%$ top-$k$ error. On the other hand, $s^*=(1,0,0)$ has $0$ top-$k$ error.

One of our main contributions is to define the top-$k$ preserving property, a necessary and sufficient property for top-$k$ optimality that solves this difficulty.
\begin{definition}[Top-$k$ preserving property]
Given $x\in\R^M$ and $y\in\R^M$, we say that $y$ is \textit{top-$k$ preserving with respect to $x$}, denoted $\spy{k}{y}{x}$, if for all $m\in[M]$, 
\begin{align*}
    x_{m} > x_{[k+1]} &\implies y_m > y_{[k+1]}\\
    x_{m} < x_{[k]} &\implies y_m < y_{[k]}.
\end{align*}
The negation of this statement is $\neg \spy{k}{y}{x}$.
\end{definition}
This is not a symmetric condition. For example, although $y=(4,3,2,1)$ is top-$2$ preserving with respect to $x=(4,2,2,1)$, $x$ is not top-$2$ preserving with respect to $y$. The following proposition and its proof illuminate the connection between top-$k$ preserving and top-$k$ optimality.
\begin{proposition}
\label{prop:opt}
 $\theta:\cX\to\R^M$ is top-$k$ Bayes optimal for any top-$k$ selector $r_k$ if and only if $\theta(X)$ is top-$k$ preserving with respect to $\eta(X)$ almost surely.
\end{proposition}
\begin{proof}
Fix $x\in\cX$ and $s\in\R^M$, with $\eta=\eta(x)$. We have 
\begin{align*}
    L_{\err_k}(s, \eta) &= \E_{y\sim\eta}[\err_k(s, y)]
    = \sum_{m\in [M]\setminus r_k(s)} \eta_m \\
    &= 1-\sum_{m\in r_k(s)} \eta_m
    \geq 1- \sum_{m=1}^k \eta_{[m]}.
\end{align*}
The last inequality holds because $|r_k(s)| = k$, so $\sum_{m\in r_k(s)} \eta_m \leq \sum_{m=1}^k \eta_{[m]}$. 
Equality occurs if and only if $\sum_{m\in r_k(s)} \eta_m = 
\sum_{m=1}^k \eta_{[m]}$. 
If equality does not hold, there exists $i\in r_k(s)$, 
$j\in [M]\setminus r_k(s)$ such that $\eta_j > \eta_i$. If $\eta_j > \eta_{[k+1]}$, then since $s_j\not\in r_k(s)$, $s_j\not> s_{[k+1]}$. If $\eta_j \leq \eta_{[k+1]}$, then $\eta_i < \eta_{[k+1]} \leq  \eta_{[k]}$. However, $s_i \not< s_{[k]}$, because $i\in r_k(s)$. Either way, $\neg \spy{k}{s}{\eta}$.  

If $\neg \spy{k}{s}{\eta}$, then there exists $i\in[M]$ such that $\eta_i > \eta_{[k+1]}$ but $s_i \leq s_{[k+1]}$, or $\eta_i < \eta_{[k]}$ but $s_i \geq s_{[k]}$. In the first case, there is an $r_k$ such that $i\not\in r_k(s)$, because there are at least $k$ indices $j\in[M]$, $j\neq i$ such that $s_j\geq s_i$. In the second case, there is an $r_k$ such that $i\in r_k(s)$, because $s_i$ is one of the top $k$ values of $s$. In either case, there is an $r_k$ such that $\sum_{m\in r_k(s)} \eta_m < \sum_{m=1}^k \eta_{[m]}$.
Thus, $L_{\err_k}(s, \eta)$ is optimal for any selector $r_k$ if and only if $\spy{k}{s}{\eta}$, i.e. $s$ is top-$k$ preserving with respect to $\eta$. 

Finally, we note that 
\begin{equation*}L_{\err_k}(\theta) = \E_{X\sim\mu}[L_{\err_k}(\theta(X), \eta(X))],
\end{equation*}
where $\mu$ is the conditional distribution of $X$. It follows that $\theta$ minimizes $L_{\err_k}(\theta)$ if and only if $\theta(X)$ minimizes $L_{\err_k}(\theta(X), \eta(X))$ almost surely. In other words, $\theta$ is a Bayes optimal predictor for any $r_k$ if and only if $\spy{k}{\theta(X)}{\eta(X)}$ almost surely.
\end{proof}

\subsection{Top-$k$ calibration}
Top-$k$ calibration characterizes when minimizing $\psi$ for a fixed $x$ leads to the Bayes decision for that $x$.
Analogous notions have been defined for binary classification, \citep{Bartlett2003ConvexityC}
multiclass classification, \citep{Zhang2004StatisticalAO}, and ranking \citep{Calauzenes2013}.
\begin{definition}[Top-$k$ calibration]
A loss function $\psi: \R^M\times\cY \to \R$ is \textit{top-$k$ calibrated} if for all $\eta\in\Delta_M$, 
\begin{equation*}
    \inf_{s\in \R^M: \neg \spy{k}{s}{\eta}}
    L_\psi(s, \eta) > \inf_{s\in\R^M}L_\psi(s, \eta) = L_\psi^*(\eta).
\end{equation*}
\end{definition}

If a minimizer $s^*$ of $L_\psi(s,\eta)$ exists, this implies that
$s^*$ must be top-$k$ preserving with respect to $\eta$. By Proposition~\ref{prop:opt}, top-$k$ calibration is necessary for minimizing $L_{\psi}$ to guarantee minimizing $L_{\err_k}$.

More generally, if $\{s^{(n)}\}$ is a sequence such that 
$L_\psi(s^{(n)}, \eta)\to\inf_s L_\psi(s,\eta)$, then it is eventually top-$k$ preserving, i.e. for all $n$ greater than some $N$, $\spy{k}{s^{(n)}}{\eta}$.


\subsection{Obtaining consistency}
We can convert top-$k$ calibration into top-$k$ consistency for all lower bounded loss functions. By Corollary 4.5 of \citet{Calauzenes2013}, since minimizing $\err_k$ is equivalent to maximizing recall at $k$, and $|\cY|=M$ is finite, if $\psi$ is continuous and nonnegative then top-$k$ calibration implies uniform calibration, which implies the existence of a surrogate regret bound $L_{\err_k}(f)-L_{\err_k}^*\leq \Gamma(L_\psi(f)-L_\psi^*)$, where $\Gamma:\R_{\geq 0}\to\R_{\geq 0}$ is continuous at 0, and $\Gamma(0)=0$. Then continuity of $\Gamma$ at 0 implies consistency: $L_{\psi}(f^{(n)})\to L_{\psi}^* \implies L_{\err_k}(f^{(n)})\to L_{\err_k}^*$. 
As an aside, we note that before we were aware of \citet{Calauzenes2013}, we proved a slightly generalized version of this result without the additional assumption that $\psi$ is continuous. 
Details are included in the appendix for completeness.
\begin{theorem}
\label{thm:consistency}
Suppose $\psi$ is a nonnegative top-$k$ calibrated loss function.
Then $\psi$ is top-$k$ consistent, i.e., for any sequence of measurable functions $f^{(n)}:\cX\to\R^M$,
we have 
\begin{equation*}
L_{\psi}(f^{(n)})\to L_{\psi}^* \implies
 L_{\err_k}(f^{(n)})\to L_{\err_k}^*.
\end{equation*}
\end{theorem}
\begin{proof}
See appendix.
\end{proof}
\section{Bregman Divergence Top-$k$ Consistent Surrogates}
Next, we outline top-$k$ consistent surrogates based on Bregman divergences.
Given a convex, differentiable function $\phi:\R^M\times\R^M\to\R$, define the Bregman divergence $D_\phi$ by 
\begin{equation}
    D_\phi(s,t) = \phi(t)-\phi(s)-\nabla\phi(s)^\top(t-s).
\end{equation}
$D_\phi(s,\cdot)$ can be interpreted as the error when approximating $\phi(\cdot)$ by the first order Taylor expansion of $\phi$ centered at $s$. Bregman divergences include squared loss and KL divergence as special cases.

Here, we present the result that any Bregman divergence composed with an {\em inverse} top-$k$ preserving function is top-$k$ calibrated. First we define inverse top-$k$ preserving functions, then give the theorem.
\begin{definition}[Inverse top-$k$ preserving function.]
Given $A\subseteq\R^M$ and $B\subseteq\R^M$,
$f:A\to B$ is \textit{inverse top-$k$ preserving} if $\forall x\in A$, $\spy{k}{x}{f(x)}$. 
\end{definition}
\begin{theorem}
\label{thm:bd}
Suppose $\phi:\R^M\to\R^M$ is strictly convex and differentiable. If $g:\R^M\to\R^M$ is inverse top-$k$ preserving, continuous, and $\Delta_M\subseteq \mathrm{range}(g)$, then $\psi:\R^M\times\cY\to\R$ defined by
\begin{equation*}
    \psi(s,y) = D_\phi(g(s), e_y)
\end{equation*}
is top-$k$ calibrated.
\end{theorem}
\begin{proof}
See Appendix.
\end{proof}
Theorem~\ref{thm:bd} is similar to one of the main results (Theorem 8) in \citet{Ravikumar2011OnNC}, except inverse order-preserving is relaxed to inverse top-$k$ preserving, the above is only a sufficient condition for top-$k$ calibration, and we make no invertibility assumptions.

\subsection{Cross entropy is top-\texorpdfstring{$k$}{} calibrated}
By Theorem~\ref{thm:bd}, the commonly used softmax with cross-entropy loss is top-$k$ calibrated:
\begin{align*}
    \Ent(s,y) &= -\ln\left(\frac{e^{s_{y}}}{\sum_{m=1}^M e^{s_m}}
    \right)
\end{align*}
can be rewritten as $\Ent(s,y) = D_\phi(g(s),e_y)$ with $\phi(x) = \sum_{m=1}^M x_m\ln x_m$ and $g(s)_m = \frac{e^{s_m}}{\sum_{i=1}^M e^{s_i}}$. $\phi$ is strictly convex and differentiable, and $g$ satisfies the assumptions of Theorem~\ref{thm:bd}.
In fact, $g$ satisfies the stronger \textit{rank preserving} condition,  
\begin{equation*}
  \forall i,j\in[M],\;  s_i > s_j \iff g(s)_i > g(s)_j.
\end{equation*}
As a result, $\Ent(s,y)$ is top-$k$ calibrated for every $k$, i.e. \textit{rank consistent}. An interesting question is whether there is a surrogate loss which does not satisfy such a strong property, and is top-$k$ calibrated for just a specific $k$. We answer in the affirmative in the sequel.

\section{Top-\texorpdfstring{$k$}{} hinge-like losses}
Hinge-like losses for top-$k$ classification have been proposed by \citet{Lapin2015TopkMS, Lapin2016LossFF}, inspired by ranking losses in \citet{Usunier:2009:ROW:1553374.1553509}, and minimized via SDCA. They note that cross entropy is competitive across datasets and values of $k$, but slight improvement is attainable with hinge losses. We list these losses as well as new ones we propose, $\psi_4, \psi_5$, in Table~\ref{tab:hinge-like}.

\begin{table}[h]
    \centering
    \vspace{-0.35cm}
    \caption{Discussed hinge-like top-$k$ loss functions along with whether they are top-$k$ calibrated. We use the notation $(x)_+ = \max\{x, 0\}$.}
    \vspace{0.2cm}
    \label{tab:hinge-like}
    \resizebox{0.5\textwidth}{!}{
    \begin{tabular}{cccc}
    \toprule
     Loss fn.    & Loss eqn. & Ref. & Calib. \\
\midrule
        $\psi_1$ & $\left(1+(s_{\setminus y})_{[k]} -s_y\right)_+$
        & \citenum{Lapin2015TopkMS, Berrada2018SmoothLF} & No\\
        $\psi_2$ & $\left(\frac{1}{k}\sum_{i=1}^k (s+\bar{\mathbf{1}}(y))_{[i]} - s_y\right)_+$
        & \citenum{Lapin2015TopkMS, Lapin2016LossFF, Lapin2018AnalysisAO} & No\\
        $\psi_3$ & $\frac{1}{k}\sum_{i=1}^k \left[(s+\bar{\mathbf{1}}(y))_{[i]} - s_y\right]_+$
        & \citenum{Lapin2015TopkMS, Lapin2016LossFF, Lapin2018AnalysisAO} & No\\
    $\psi_4$ & $\left(\frac{1}{k}\sum_{i=1}^k 
    (1+(s_{\setminus y})_{[i]}) - s_y\right)_+.$
        & New & No\\
        $\psi_5$ & $\left(1+s_{[k+1]}-s_y, 0\right)_+$
        & New & Yes\\
\bottomrule
    \end{tabular}
    }
\end{table} 
The motivation of these losses is as follows. $\psi_1$ is a generalization of multiclass SVM \citep{Crammer2001OnTA}. $\psi_2$ and $\psi_3$ are convex upper bounds on $\psi_1$.  \\
We propose $\psi_4$ as a tighter convex upper bound on $\psi_1$ and $\psi_5$ as the tightest bound on $\err_k$ of all, and the only top-$k$ calibrated loss. Next, we show that $\psi_5$ is top-$k$ calibrated and the rest, $\psi_1,\psi_2,\psi_3,\psi_4$, are not top-$k$ calibrated. These facts are not in previous literature. \\
\subsection{Characterization of hinge-like losses}
We compute the minimizers of the expected loss $L_{\psi_1}(s,\eta) = \E_{y\sim\eta}[\psi_1(s,y)]$ given a conditional distribution $\eta\in\Delta_M$. Though we arrive at inconsistency,
our results also indicate that if $\eta$ is from the restricted probability simplex $\{\eta \in \Delta_M \mid \eta_{[k]} > \sum_{i=k+1}^M \eta_{[i]}\}$, $\psi_1$ is calibrated/consistent.
\begin{theorem}[Abridged]
\label{psi1thm}
Let $\eta\in\Delta_M$ and suppose $\eta_1\geq \eta_2\geq \ldots \geq \eta_M$. Then, 
\begin{align*}
    &\eta_k \geq \sum_{i=k+1}^M \eta_i
    \implies [\overbrace{1\; 1\,\ldots\, 1}^{k-1}\; 0 \; 0\,\ldots\,0]\\
    &\qquad\qquad\qquad\qquad\in \argmin_{s} L_{\psi_1}(s, \eta) \\
    &\eta_k \leq \sum_{i=k+1}^M \eta_i
    \implies [\overbrace{1\; 1\,\ldots\, 1\; 1}^{k}\; 0 \,\ldots\,0]\\
    &\qquad\qquad\qquad\qquad\in \argmin_{s} L_{\psi_1}(s, \eta).
\end{align*}

\end{theorem}
\begin{proof}
See appendix for the exact set of minimizers when $\eta$ has no zero entries, and proof.
\end{proof}
This implies that $\psi_1$ is not top-$k$ calibrated: if $\eta_1 > \ldots > \eta_M$ then in the first case of the above theorem, $s^*$ is not top-$k$ preserving with respect to $\eta$: for any $m\in\{k+1,\ldots, M\}$, $\eta_m < \eta_k$, and yet $s_m \not< s_{[k]} =0$.
Yet, $s^*$ is a minimizer of $L_{\psi_1}(s, \eta)$, so $\psi_1$ is not top-$k$ calibrated.

The following proposition implies that $\{\psi_2, \psi_3, \psi_4\}$ are not top-$k$ calibrated, and are thus inconsistent. 
\begin{proposition}
\label{prop:inconsistent}
For any $\psi\in\{\psi_2, \psi_3, \psi_4\}$,
if $\sum_{m=k+1}^M \eta_{[m]} > \frac{k}{k+1}$, we have $0\in\argmin_s L_\psi(s, \eta)$, and thus $L^*_\psi(\eta) = \min_s L_\psi (s, \eta) = L_\psi(0, \eta) = 1$. 
\end{proposition}
\begin{proof}
See Appendix.
\end{proof}
To show this leads to inconsistency, take $\eta = (1/8, 1/8, 1/12, 1/12, \ldots, 1/12) \in \Delta_{11}$ with $k=2$. $\eta$ satisfies $\sum_{i=k+1}^M \eta_{[i]} = 
\frac{3}{4} > \frac{2}{3} = \frac{k}{k+1}$, so the optimal is $s^*=0$. But, $s^*$ is not top-$k$ preserving wrt $\eta$. This implies that $\psi\in\{\psi_2, \psi_3, \psi_4\}$ is not top-$k$ calibrated. 
\begin{proposition}
\label{prop:psi5}
$\psi_5:\R^M\times\cY\to\R$  is top-$k$ calibrated. 
\end{proposition}
\begin{proof}
See Appendix. Note since $\psi_5$ is bounded below, by Theorem~\ref{thm:consistency}, it is top-$k$ consistent.
\end{proof}
\subsection{Conjecture on the lack of convex hinge losses}
Generally, a hinge loss can be considered to have the form 
\begin{equation*}
    \psi(s,y) = \max\{w^\top P(f(s)) f(s-s_y\mathbf{1}), 0\}
\end{equation*}
where $f$ is an affine function (or may also contain a $(\cdot)_+$) and $P$ is a permutation matrix depending on $f(s)$. $w$ is a fixed vector. For example, for $\psi_2$, we have $f(s) = s+\bar{\mathbf{1}}(y)$, $P$ the sort matrix, and $w=1/k$ for the first $k$ entries. $\psi_3$ and $\psi_4$ are similar. If we assume $\psi$ is convex, we must have $P$ be the sorting matrix and $w$'s entries in decreasing order \citep{Usunier:2009:ROW:1553374.1553509}. Intuitively, the closest we can get to being top-$k$ calibrated is when $w$'s nonzero entries are equal; this leads to essentially the existing hinge loss surrogates, which are uncalibrated. Thus, we conjecture that no convex, piecewise affine loss is top-$k$ calibrated.

\section{Linear (in)consistency}
Until now, we have been discussing consistency with respect to all measurable functions, as is standard. We may instead consider consistency with respect to a restricted function class $\cF$. This type of consistency was explored for $k=1$ in \citet{Long13}. Time of 
\begin{definition}[$\cF$-consistency]
$\psi:\R^m\times\cY\to\R$ is $\cF$ top-$k$ consistent (or $\cF$-consistent) if 
\begin{equation*}
    L_{\psi}(f_n) \to \inf_{f'\in \cF} L_{\psi}(f') 
    \implies 
    L_{\err_k}(f_n)\to \inf_{f'\in\cF} L_{\err_k}(f'),
\end{equation*}
where $(f_n)_{n=1}^{\infty}$ is a sequence of functions $\cX\to\R^M$ in $\cF$. If no conditions or set of distributions are specified, $\cF$-consistent means the above holds for every probability distribution over $\cX\times\cY$.
\end{definition}
Previously, the infimum with respect to the scoring function was over all measurable functions, but in practice, we minimize using some function class, e.g., functions computed by a neural net architecture.\\
$\cF$-consistency seems much more difficult to analyze than consistency because we may no longer decompose the risk into $L(f(x), \eta(x))$ for each $x$, as $f$ cannot vary its outputs arbitrarily. Furthermore, if $\cX=\R^d$ and $\cF$ consists of linear functions, $\cF$-consistency of a convex $\psi$ suggests $\mathsf{P}=\mathsf{NP}$, due to the efficiency of convex minimization and the $\mathsf{NP}$-hardness of finding a linear separator which maximizes accuracy \citep{BenDavid2003}.\\
On the other hand, letting $L^*(\cF) = \inf_{f'\in\cF} L(f')$, as long as $L_{\psi}^*(\cF) = L_\psi^*$ and $L_{\err_k}^*(\cF) = L_{\err_k}^*$, top-$k$ consistency implies $\cF$-top-$k$ consistency because
\begin{align}
\label{eq:fcimp}
    &L_{\psi}(f_n)\to L_{\psi}^*(\cF) \implies 
    L_{\psi}(f_n)\to L_{\psi}^*\nonumber\\
    \implies &
    L_{\psi}(f_n)\to L_{\err_k}^* \implies 
    L_{\psi}(f_n)\to L_{\err_k}^*(\cF).
\end{align}
Furthermore, we can answer easier questions about $\cF$-consistency by making additional assumptions, e.g. top-$k$ separability. If there is a top-$k$ separator, i.e. a predictor with perfect top-$k$ accuracy, then does our algorithm (i.e., minimizing a surrogate loss) find it? Despite $\mathsf{NP}$-hardness in general, if a linear separator exists for a binary classification problem, one can be found efficiently, so it seems appropriate to ask an analogous question for top-$k$ separability in the context of surrogate losses.  
\begin{proposition}
\label{prop:fcon}
Let $\cX=\R^d$ and $\cF = \{x\mapsto Wx: W\in\R^{M\times d}\}$. Then if we consider \textit{top}-$k$ \textit{separable} probability distributions over $\cX\times\cY$, i.e.
$L_{\err_k}^*(\cF) = 0 = L_{\err_k}^*$, then: 
\begin{enumerate}
\item If $k=1$, $\Ent$ is $\cF$-consistent.
    \item If $d\geq 3, M\geq 3$, and $k=2$, $\Ent$ is not $\cF$-consistent.
    \item $\psi_1$ and $\psi_5$ are $\cF$-consistent.
\end{enumerate}
\end{proposition}
The above proposition says the answer is yes for $\psi_1$ and $\psi_5$, and generally no for $\Ent$ unless $k=1$. 
To see Propopsition~\ref{prop:fcon}.1, note that top-1 separability means $\exists W\in\R^{M\times d}$ where $\Pr[(Wx)_y > (Wx)_{[2]}] = 1$. Then, w.p. 1 over $x,y$,
\begin{align*} \Ent(cWx, y) &= \log\left(1+\sum_{m\neq y}e^{c((Wx)_m - (Wx)_y)}\right)\\ &\xrightarrow{c\to\infty}{\log(1)}=0.
\end{align*}
 Thus, $\Ent^*(\cF)=0=\Ent^*$ and we have $\cF$ consistency  by \eqref{eq:fcimp}. Note we cannot spply this "scaling to 0 loss" argument for $\Ent$ when $k\geq 2$.
The rest of the proof is in the appendix. 
\section{A convex, differentiable loss function}
While we achieved top-$k$ calibration for a specific $k$ with the $\psi_5$ loss, one might wonder whether this is possible with a convex, differentiable loss function. In some sense, because of the case of $M=2$, one would expect that if a convex, differentiable loss function is top-$k$ calibrated for some $k<M$, then it is top-$k$ calibrated for all $k'$. In \citet{Bartlett2003LargeMC}, it was proven that a convex margin function just needs to have negative derivative at 0 to be binary consistent, raising the question of whether a similar claim can be made when the number of labels increases. The increase in number of directions the score vector can travel makes the question much harder to answer. \\
It turns out that this is partially true, and partially untrue. It is true in the sense of the following theorem:
\begin{theorem}
Suppose $\psi(s,y)$ is convex and differentiable for each $y\in [M]$, and moreover if we think of $\Psi(s)$ as the $M$ length vector whose entries are $\psi(s,y)$, symmetric in the sense of $\Psi(Ps)=P\Psi(s)$ for all permutation matrices $P$. Then, if $\psi$ is top-$k$ calibrated for some $k<M$, it is top-$k'$ calibrated for all $k'\leq k$. 
\end{theorem}
\begin{proof}
Let $e_i$ denote the $i$th coordinate basis vector.
Suppose that $s^*$ minimizes $L_{\psi}(s, \eta)=\ip{\eta, \Psi(s)}$. 
Suppose that $i,j$ are in the arguments of the top-$k$ entries of $\eta$, and $\eta_i > \eta_j$. 
Define $\eta^{\flat}$ as $\eta$ but with $\eta_i$ and $\eta_j$ replaced with their average. For a large enough $\delta>0, \tilde\eta = \eta^{\flat} + \delta(e_i-e_j)\in\Delta_M$ has $j$ no longer in the top-$k$ entries of $\tilde\eta$. Suppose $\tilde s$ minimizes $\ip{\tilde\eta, \Psi(s)}$ and $s^{\flat}$ minimizes $\ip{\eta^{\flat}, \Psi(s)}$. We have 
\begin{align*}
    0 &> \ip{\tilde\eta, \Psi(\tilde s)-\Psi(s^{\flat})} \\
    &\geq \tilde\eta^\top\nabla \Psi(s^{\flat})(\tilde s - s^{\flat}) \\
    &= \delta(e_i-e_j)^\top\nabla\Psi(s^{\flat})(\tilde s-s^{\flat}) \\
    &= \delta(\nabla\psi(s^{\flat}, i) - \nabla\psi(s^{\flat}, j))^\top(\tilde s-s^{\flat})\\
    &= \delta(a-b)(\tilde s_i-\tilde s_j).
\end{align*}
The second line is by convexity of $\Psi$, the third line is by optimality of $s^{\flat}$ for $\eta^{\flat}$. The last line uses symmetry of $\Psi$: since $s^{\flat}_i=s^{\flat}_j$ (follows from convexity of $\Psi$ and $\eta^{\flat}_i=\eta^{\flat}_j$), the $i$th and $j$th gradients are equal to each other, except their $i$th and $j$th entries, $a$ and $b$, are swapped.\\
Since $\psi$ is top-$k$ preserving and $j$ is no longer in the top-$k$ entries of $\tilde\eta$, we have $\tilde s_i > \tilde s_j$. Thus, $a<b$. Notice that we can replace $\tilde s$ with $s^*$ and everything in the chain holds -- $s^{\flat}$ is not a minimizer of $\ip{\eta, \Psi(s)}$, because $\nabla\ip{\eta, \Psi^{s^{\flat}}}(\tilde s-s^{\flat}) = \eta\nabla\Psi(s^{\flat})(\tilde s-s^{\flat}) = \delta'(a-b)(\tilde s_i-\tilde s_j) <0$. Therefore, $s^*_i > s^*_j$, as desired.
\end{proof}
However, it is untrue in that we can exhibit a convex, differentiable, symmetric loss which is top-1 calibrated but not top-2, 3, 4, \ldots calibrated. It is shown below:
\begin{align}
\label{eq:convdiff}
    &\Psi^{CD}(s,y) = \log(1+\exp(-s_y))\nonumber\\
    &+ \sum_{i\neq y}\left(s_i - \frac{1}{M-1}\sum_{j\neq y} s_j\right)^2 + \sum_{i\neq y} s_i^2
\end{align}
To show $\Psi^{CD}$ is not calibrated for $1<k<M$, we run gradient descent on 
$L_{\psi^{CD}}(s, \eta)=\ip{\eta, \Psi^{CD}(s)}$ with $\eta=[0.01, 0.02, 0.03, 0.04, 0.9]$ 
and reach an optimum of $[0.0114226, 0.011404, 0.011385, 0.011365, 0.470880]$. While the most probable class got the highest score, the scores of the others are reversed relative to probability. To see why this happens intuitively, the presence of the logistic loss makes $s_5$ at optimum much higher than the others, since $\eta_5$ is by far the largest. Now for $y\neq 5$, the best way to decrease the loss $\psi(s, y)$ is to increase $s_j,\, j\neq y$, because the mean is being blown up by $s_5$, and $s_y$ is deliberately excluded from the mean differences.
\begin{theorem}
$\Psi^{CD}$ is top-1 calibrated. 
\end{theorem}
\begin{proof}
Consider $s\in\R^M$ and WLOG suppose $s\geq0$, and $s_1$ is the maximum entry. We have 
\begin{align*}
    \psi(s,2) - \psi(s,1) &\geq (M-1)(\Var(X) + \E[X^2])\\
&\;\;-  (M-1)(\Var(Y)-\E[Y^2])
\end{align*}
Where $X$ is uniform over $\{s_1, s_3,s_4,\ldots, s_M\}$ and $Y$ is uniform over $\{s_2, s_3, s_4,\ldots, s_M\}$. For the remainder of the proof, we let $n:=M-1$ for brevity. Showing $\psi(s,2)-\psi(s,1)> 0$ may be done by showing
\begin{align*}
    &\Var(X)+\E[X^2]-\Var(Y)-\E[Y^2]\\
   &\quad = 2(\E[X^2]-\E[Y^2]) - (\E[X]^2-\E[Y]^2)
    > 0.
\end{align*}
Letting $m=\sum_{i=2}^M s_i$, we have
\begin{align*}
    &2(\E[X^2]-\E[Y^2]) - (\E[X]^2-\E[Y]^2)\\
&\quad=\frac{2(s_1^2-s_2^2)}{n}-\frac{1}{n^2}[(m+s_1-s_2)^2 - m^2)]\\
&\quad=\frac{2(s_1^2-s_2^2)}{n}-\frac{2m(s_1-s_2)
+ (s_1-s_2)^2}{n^2}\\
&\quad=\frac{2(s_1^2-s_2^2)}{n}-\frac{2(\frac{m}{n}+\frac{s_1-s_2}{n})(s_1-s_2)}{2n}
\end{align*}
To complete the proof we just need to show 
that $\frac{m}{n}+\frac{s_1-s_2}{n} < s_1+s_2$. This is equivalent to showing that 
$m< (n-1)s_1 + (n+1)s_2$. But this is true;  $\sum_{i=3}^M s_i \leq (M-2)s_1 + ns_2$, as each $s_i< s_1$ and $s_2\geq 0$. 
This proves that $\psi(s,2)-\psi(s,1) > 0$ if $s_2<s_1$.
\end{proof}
\section{Synthetic Data Experiments}
Here we describe experiments comparing an assortment of top-$k$ surrogate loss functions on synthetic data, to see how their behavior compares with reference to the theory. One synthetic experiment empirically showcases  the inconsistency of $\psi_1,\psi_2,\psi_3,\psi_4$ and consistency of $\psi_5$. A second and third experiment flesh out the behavior of the losses in different regimes. we also employ the classic cross entropy loss $\Ent$, and the following truncated cross entropy losses:
\begin{align*}
\Enta(s,y) &= -\ln g(s)_y\\
\Entb(s,y) &= -\ln g(s)_y + \sum_{i=1}^M g(s)_i - 1
\end{align*}
with $g(s)_j = \frac{\exp(s_j)}{\exp(s_j) + \sum_{i=k}^{M-1} \exp(s_{\setminus j})_{[i]}}$. $\Enta$ was proposed in \citet{Lapin2016LossFF}, and we propose $\Entb$ by restoring the terms dropped from the Bregman Divergence by $\Enta$. Since $g$ is inverse order preserving, by Theorem~\ref{thm:bd} in fact $\Entb$ is top-$k$ calibrated for every $k$. 

We use Pytorch to implement each loss and use them to train on synthetic data. A machine with an Intel Core i7 8th-gen CPU with 16GB of RAM was used.

The first synthetic data experiment we conduct highlights the consistency/inconsistency of the top-$k$ hinge losses. By Proposition~\ref{prop:inconsistent}, if the $k+1$ least likely classes altogether have a probability of occurring greater than $\frac{k}{k+1}$, the predictions made by $\psi_2, \psi_3, \psi_4$ equal a constant vector, and by Theorem~\ref{psi1thm}, $\psi_1$ will assign a value of $c+1$ to the $k-1$ most probable classes and $c$ to the rest. This behavior is inconsistent. On the other hand, $\psi_5$, which is top-$k$ consistent, will still assign values of $c+1$ to the $k$ most probable classes, and $c$ to the rest. 

We construct training data which matches the above setting. The data contains $68$ data points with each input data point equal to the zero vector in $\R^2$. Each class in $\{1, 2\}$ is assigned to 10 data points, and each class in $\{3,4,5,6,7,8\}$ is assigned to 8 data points. 
We set $k=2$ so that $\sum_{i=k+1}^M \eta_{[i]} = \frac{48}{68} > \frac{2}{3}$, as described in Proposition~\ref{prop:inconsistent}. 
We train our neural architecture on the data using batch gradient descent, setting the loss of the last layer to be each of $\{\psi_1,\ldots,\psi_5\}$ with $k=2$. For each classifier obtained, we evaluate the top-$2$ error on the training set. This is repeated for 100 trials to ensure the robustness of our results.

One may surmise 
that even if the theoretical minimizers for a loss are not top-$k$ Bayes optimal, they may be effective in practice due to the optimization process. For example, the learned classifier for $\psi_2$ could output a vector close to $0$, but with the first two entries minutely greater than the rest.
Interestingly, this is not the case: the returned classifiers for $\psi_2, \psi_3, \psi_4$ essentially pick randomly amongst the 8 possible classes. The classifier returned by $\psi_1$ chooses one of $\{0, 1\}$, and randomly picks from the rest of the classes. Finally, the classifier returned by $\psi_5$ returns the Bayes decision rule, $\{0, 1\}$. These results closely align with the theoretical optima of these losses.

We report average top-$2$ accuracy over the 100 trials in Table~\ref{tab:synexp1}. For reference, predicting $\{0,1\}$ yields a top-$2$ accuracy of $\frac{20}{68} = 0.294$, predicting one of them gives $\frac{18}{68} = 0.265$, and predicting none of them gives $\frac{16}{68} = 0.235$. Examples of score vectors returned by each loss are in the Appendix. We note that the neural net trained with $\psi_5$ predicts $\{0, 1\}$ every trial. 
\begin{table}[!h]
    \vspace{-0.4cm}
    \centering
    \caption{Results for Top-2 accuracy on the synthetic dataset demonstrating consistency/inconsistency of hinge-like losses. Averaged over 100 trials.}
    \vspace{0.2cm}
    \label{tab:synexp1}
    \begin{tabular}{cccccc}
    \toprule
         & $\psi_1$ & $\psi_2$& $\psi_3$& $\psi_4$& $\psi_5$\\
         \midrule
Top-$2$:  & 0.2671 & 0.2515 & 0.2500 & 0.2468 & 0.2941\\
\bottomrule
    \end{tabular}
\end{table}

To investigate a more interesting and realistic example,
we also conduct the following synthetic experiment. Given an input $N$, we randomly sample from a $d$ dimensional Gaussian until we find $N$ vectors which are all at least $c\sqrt{d}$ apart from each other in $\ell_2$ distance. Then, we assume there are $M$ classes, where $M$ is a parameter. For each class, we randomly select $K$ of the $N$ means, and then generate a random probability distribution over the $K$ means. Then, we sample $L$ points from the class, by randomly picking a mean according to the probability distribution and sampling from a Gaussian centered there. This models a situation where labels have overlapping distributions.

 We set $d=2, c=2, K=5, L=40$ and vary $N$ in $\{10, 50, 100\}$ to generate the training set. We generate a test set using the same Gaussians and classes with $l=7$. Results are shown in Table \ref{tab:synth2}, averaged over 10 trials of generating the data followed by training and evaluation of classifiers on the test set. We optimize with Adam for 500 epochs, using a learning rate of $0.1$ and full batch. 

Usually, cross entropy dominates other losses in performance. However, in this experiment, due to the overlapping nature of the label distributions, and the function class being restricted to linear predictors, cross entropy actually does notably worse than certain losses which particularly perform well in this scenario -- $\psi_1, \psi_5$, and $\Ent_{\mathrm{Tr}_1}$. This can be viewed as an empirical validation of our results on the linear-restricted inconsistency of cross entropy and consistency of $\psi_1$ and $\psi_5$. Furthermore, it light of our discussion of the relationship between convexity and calibration, it is interesting that specifically the nonconvex losses do well in this scenario.

Another interesting phenomenon we observe is that $\psi_5$ is in a sense robust to its setting of $k$. While the performance of $\psi_1$ and $\Ent_{\mathrm{Tr}_1}$ degrade noticeably for top-5 accuracy in the $N=100$ case, the performance of $\psi_5$ stays about the same. This is in keeping with $\psi_5$ being more lenient, not caring as much as long as the top-$k$ error is 0. 
\vspace{-0.2cm}
\begin{table}[!h]
\caption{Results of the second synthetic experiment. Superscript on loss function denotes which $k$ is taken in the loss. We try out both $k=5$ and $k=4$ for the $N=100$ case.}
\vspace{0.2cm}
\begin{tabular}{ccccc}
\toprule
\label{tab:synth2}
&\multicolumn{2}{c}{$N=10$}  
&\multicolumn{2}{c}{$N=50$}\\
{} &  Top-5 &    Acc &  Top-5 &    Acc \\
\midrule
$\Ent$                 &      0.699 &  $\mathbf{0.212}$ & 0.755 &  $\mathbf{0.267}$ \\
$\psi_1^5$             &      $\mathbf{0.737}$ &  0.120 & 0.869 &  0.134\\
$\psi_2^5$             &      0.639 &  0.189 & 0.734 &  0.245\\
$\psi_3^5$             &      0.649 &  0.191 &  0.741 &  0.241 \\
$\psi_4^5$             &      0.651 &  0.185 & 0.740 &  0.205\\
$\psi_5^5$             &      0.726 &  0.117 & $\mathbf{0.880}$ &  0.149\\
$\Ent^5_{\mathrm{Tr}_1}$ &      0.711 &  0.125 & 0.879 &  0.118\\
$\Ent^5_{\mathrm{Tr}_2}$ &      0.636 &  0.169 & 0.656 &  0.196\\
\bottomrule
\end{tabular}

\begin{tabular}{ccccc}
\toprule
&\multicolumn{2}{c}{$N=100,\, k=5$}  
&\multicolumn{2}{c}{$N=100,\, k=4$}\\
{} &  Top-5 &    Acc &  Top-5 &    Acc \\
\midrule
$\Ent$                 &      0.763 &  $\mathbf{0.242}$ &  0.761 &  $\mathbf{0.224}$\\
$\psi_1^k$             &      $\mathbf{0.896}$ &  0.131 & 0.834 &  0.144 \\
$\psi_2^k$             &      0.734 &  0.236 & 0.721 &  0.236 \\
$\psi_3^k$             &      0.711 &  0.214 & 0.722 &  0.219\\
$\psi_4^k$             &      0.744 &  0.210 & 0.720 &  0.201 \\
$\psi_5^k$             &      0.884 &  0.124 & $\mathbf{0.868}$ &  0.123\\
$\Ent^k_{\mathrm{Tr}_1}$ &      0.892 &  0.111 & 0.857 &  0.136\\
$\Ent^k_{\mathrm{Tr}_2}$ &      0.686 &  0.169 & 0.726 &  0.221\\
\bottomrule
\end{tabular}
\end{table}

We also model more separated probability distributions. We generate $N$ means as described earlier. For each mean, we sample $kl$ points from the Gaussian centered at the vector with covariance matrix $I\in\R^{d\times d}$. Each set of $kl$ points is divided into $k$ classes of $l$ points each. The top-$k$ error is necessary to achieve 0 error because each Gaussian center spawns $k$ classes that are indistinguishable from each other.

 We set $d=5, c=2, k=5, l=20$ and vary $N$ in $\{10, 50, 100\}$ to generate the training set. We generate a test set using the same Gaussians and classes with $l=7$. Results are shown in Table \ref{tab:synth3}, averaged over 10 trials of generating the data followed by training and evaluation of classifiers on the test set.
 
 We find that while on this more conventional dataset, $\Ent$ dominates, the newly proposed $\psi_4,\, \psi_5$ do the best among the other losses.
\begin{table}[!h]
\caption{Third set of synthetic experiments, each value averaged over 10 trials. $N$ is the number of Gaussian centers. Superscript on top-$k$ losses indicates the value of $k$ for that loss. Top-5 is top-5 accuracy=$1-\err_5$, Acc. is accuracy, $\Delta_1$ = test loss - test top-5 error. N/A means not computed due to numerical instability.}
\vspace{0.2cm}
\label{tab:synth3}
\centering
\resizebox{0.5\textwidth}{!}{
\begin{tabular}{ccccccc}
\toprule
&\multicolumn{2}{c}{$N=10$}
&\multicolumn{2}{c}{$N=50$}
&\multicolumn{2}{c}{$N=100$}\\
&Top-5&Acc.&Top-5&Acc.&Top-5&Acc.
\\\midrule
$\mathrm{Ent}$&0.932&$\mathbf{0.196}$&$\mathbf{0.914}$&$0.180$&$\mathbf{0.888}$&$\mathbf{0.183}$\\
$\psi_1^5$&0.844&0.146&0.720&0.132&0.613&0.126\\
$\psi_2^5$&0.918&0.187&0.784&0.179&0.651&0.162\\
$\psi_3^5$&0.924&0.192&0.784&0.180&0.640&0.160\\
$\psi_4^5$&$\mathbf{0.933}$&0.186&0.812&$\mathbf{0.181}$&0.661&0.157\\
$\psi_5^5$&$0.874$&0.179&0.801&0.172&0.695&0.146\\
$\mathrm{Ent}_{\mathrm{Tr}_1}^5$&0.803&0.129&$0.815$&0.153&0.649&0.127\\
$\mathrm{Ent}_{\mathrm{Tr}_2}^5$&0.802&0.177&N/A&N/A&N/A&N/A\\\bottomrule
\end{tabular}}

\end{table}

\section{Conclusion}
We laid out a theoretical framework for the consistency of surrogate losses used in top-$k$ classification, by defining top-$k$ preserving-ness and top-$k$ calibration. 

Our subsequent results on the calibration of losses possessing a form involving Bregman divergences and on the inconsistency of various hinge losses, in constrast to the consistency of a new one we propose, chart some of the consistency landscape of top-$k$ surrogate losses. 

We further develop the theory of top-$k$ consistency by exploring a practically relevant extension: consistency restricted to a particular function class. Furthermore, we analyze the relationship of convexity to top-$k$ calibration. With hinge losses, convexity seems antithetical to top-$k$ calibration, and when differentiability is added, top-$k$ calibrated losses are nice, up to a certain limit that is demonstrated via an interesting counterexample.

Future directions include investigating which losses generalize well in the context of top-$k$ classification, as this is the natural and practical progression of the inherent infinite sample assumption of consistency, and determining consistency when restricted to deep learning function classes.

\bibliography{bib}
\bibliographystyle{icml2020}

\newpage
\text{}
\newpage
\section{Additional Proofs}
In addition to providing proofs not in the main text in chronological order, we restate what is being proved for convenience.
\begin{clemma}{For Theorem~\ref{thm:consistency}}
\label{lem:cont}
Let $\psi:\R^M\times\cY\to[0, \infty)$ be a nonnegative loss function.
$L_{\psi}^*:\Delta_M\to\R$ defined by $L_{\psi}^*(\eta) = \inf_{s\in\R^M} 
\sum_{i=1}^M \eta_i\psi(s, i)$ is continuous.
\end{clemma}
\begin{proof}
First, note that $L_{\psi}^*$ is concave, because it is a pointwise 
infimum of affine functions of $\eta$. Also, it is finite valued,
because $\psi$ is lower bounded (thus $L_{\psi}^*(\eta)>-\infty$)
and clearly $L_{\psi}^*(\eta) < \infty$. \\
By Theorem 10.2 of \citet{rockafellar-1970a}, any concave function taking finite
real values on a locally simplicial subset $S\subseteq\R^M$ is 
lower semicontinuous. That is, for all $x\in S$ and sequences 
$\{x^{(n)}\}$ converging to $x$, $f(x)\leq\lim_{n\to\infty}
f(x^{(n)})$ if the limit on the right exists. \\
$\Delta_M$ is locally simplicial (it is the probability simplex) and 
$L_{\psi}^*$ satisfies the assumptions,
 so $L_{\psi}^*$ is lower semicontinuous. \\
Now we just need to show upper semicontinuity, which can be stated as:
for any $\epsilon > 0, \eta\in\Delta_M$, 
there exists $\delta > 0$ where for all $\eta'\in\Delta_M,$ 
$\|\eta'-\eta\|_2\leq \delta$ implies $L_{\psi}^*(\eta')
\leq L_{\psi}^*(\eta)+\epsilon$. \\
Let $\eta\in\Delta_M, \epsilon>0$.
 Choose $s$ so that $L_{\psi}(s, \eta) \leq 
L_{\psi}^*(\eta)+\epsilon/2$, which is possible by definition of $L^*$.
Now set $\delta = \epsilon\left(2\max\left\{\sqrt{
\sum_{i=1}^M\psi(s, i)^2}, 1\right\}\right)^{-1}
$ (taking the max with 1 is to avoid a zero in the denominator),
 and suppose $\eta'\in\Delta, \|\eta-\eta'\|_2\leq \delta$. We have,
\begin{align*}
L_\psi^*(\eta') &\leq L_\psi(s,\eta') = \sum_{i=1}^M \eta_i'\psi(s, i) \\
&= \sum_{i=1}^M\eta_i\psi(s,i) + \sum_{i=1}^M (\eta'_i-\eta_i)\psi(s,i)\\
&\leq L_{\psi}^*(\eta)+\epsilon/2 + \|\eta'-\eta\|_2
\sqrt{\sum_{i=1}^M \psi(s,i)^2} \\
&\leq L_{\psi}^*(\eta) + \epsilon/2 
+ \epsilon/2 = L_{\psi}^*(\eta) + \epsilon.
\end{align*}
The first inequality is by definition of $L^*$, and the second inequality
uses the Cauchy-Schwartz inequality.
Therefore, $L^*$ is upper semicontinuous. Since it is also lower 
semicontinuous, it is continuous.
\end{proof}
\begin{cthm}{\ref{thm:consistency}}
Suppose $\psi$ is a nonnegative top-$k$ calibrated loss function.
Then $\psi$ is top-$k$ consistent in the sense that
 for any sequence of measurable functions $f^{(n)}:\cX\to\R^M$,
we have 
\begin{equation*}
L_{\psi}(f^{(n)})\to L_{\psi}^* \implies
 L_{\err_k}(f^{(n)})\to L_{\err_k}^*.
\end{equation*}
\end{cthm}
\begin{proof}
We place top-$k$ classification in the abstract decision model in
Appendix A. of \citet{Zhang2004StatisticalAO} with output-model space $\cQ=\Delta_M$,
decision space $\cD$ equal to the set of subsets of $[M]$ of size 
$k$, and estimation-model space $\Omega=\R^M$. The risk function is 
the top-$k$ error and the decision rule is 
equal to $r_k$, the top-$k$ thresholding operator. \\
By Corollary 26 of \citet{Zhang2004StatisticalAO} we just need to show that for 
any $\epsilon >0$,
\begin{align*}
\Delta H(\epsilon) = \inf\left\{\Delta L_{\psi}(s, \eta) 
\mid  
 \Delta L^*_{\err_k}(s, \eta) 
\geq \epsilon \right\}>0,
\end{align*}
where $\Delta L(s,\eta) := L(s,\eta) - L^*(\eta)$. 
In other words, we need to show that given any $\epsilon>0$, there is 
a $\delta >0$ such that $\Delta L_{\err_k}(s, \eta) \geq \epsilon$ 
implies $\Delta L_{\psi}(s, \eta) \geq \delta$. \\
Proof by contradiction.
Given $\epsilon >0$, assume there does not exist $\delta > 0$ such 
that the above holds. Then, there is a sequence $\{s^{(n)}, \eta^{(n)}\}$ 
such that $\Delta L_{\err_k}(s^{(n)}, \eta^{(n)}) \geq \epsilon$ for 
all $n\in\N$ and yet $\Delta L_{\psi}(s^{(n)}, \eta^{(n)}) \to 0$. 
Since $\eta^{(n)}$ comes from a compact set $\Delta_M$, we may assume 
that $\eta^{(n)} \to \eta$ without loss of generality, since otherwise we could take a convergent subsequence.\\
 We will show that 
$\Delta L_{\psi}(s^{(n)}, \eta) \to 0$, which provides a contradiction in the following.
Because $\psi$ is top-$k$ calibrated,
 $s^{(n)}$ is top-$k$ preserving 
with respect to $\eta$ for all $n$ greater than some $N$. This means there exists $N$ where $\Delta L_{\err_k}(s^{(n)}, \eta) = 0$ for all $n > N$, i.e. $L_{\err_k}(s^{(n)}, \eta) = L^*_{\err_k}(\eta)$. By continuity of $L^*_{\err_k}$, there exists $N'$ such that $|L^*_{\err_k}(\eta^{(n)}) - L_{\err_k}^*(\eta)| < \frac{\epsilon}{2}$ for all $n>N'$. But this means $\Delta L_{\err_k}^*(s^{(n)}, \eta^{(n)}) < \frac{\epsilon}{2}$ for $n>\max\{N, N'\}$, a contradiction. \\
Since $\Delta L_\psi(s^{(n)}, \eta^{(n)})\to 0$,
for any $\epsilon'>0$, 
there exists $N>0$ such that for all $n>N$, we have 
\begin{equation*}
\vert L_\psi(s^{(n)}, \eta^{(n)}) - L_\psi^*(\eta^{(n)})\vert
\leq \epsilon'/2.
\end{equation*}
Moreover, since $L_\psi^*$ is continuous by Lemma \ref{lem:cont} and $\eta^{(n)}\to\eta$,
there exists $N'>0$ such that for all $n>N'$, we have 
\begin{align*}
\lvert L_\psi^*(\eta^{(n)}) - L_\psi^*(\eta) \rvert \leq \epsilon'/2.
\end{align*}
Then, for all $n>\max\{N, N'\}$, 
\begin{align*}
\vert L_\psi(s^{(n)}, \eta^{(n)}) - L_\psi^*(\eta)\vert  &\leq 
\vert 
L_\psi(s^{(n)}, \eta^{(n)}) -L_\psi^*(\eta^{(n)})\vert\\
 &+ \vert L_\psi^*(\eta^{(n)}) - L_\psi^*(\eta)\vert \leq \epsilon'.
\end{align*}
Since $\epsilon'$ was arbitrary, we have $L_\psi(s^{(n)}, \eta^{(n)})
\to L^*_\psi(\eta)$. \\
Now we extend to $L_\psi(s^{(n)}, \eta) \to L_\psi^*(\eta)$ by showing that $L_\psi(s^{(n)}, \eta^{(n)})$ is close to 
$L_\psi(s^{(n)}, \eta)$. Given any $\epsilon'>0$,
let $N$ be such that for all $n>N$, $L_\psi(s^{(n)},
\eta^{(n)}) - L_\psi^*(\eta) \leq \epsilon'$. Then we have for all 
$n > N$
\begin{equation*}
L_\psi(s^{(n)}, \eta^{(n)}) - L_\psi(s^{(n)}, \eta) 
\leq L_\psi(s^{(n)}, \eta^{(n)}) - L_\psi^*(\eta) \leq \epsilon'.
\end{equation*}
Let $I$ be the support of $\eta$. For every $i\in I$, 
$\{\psi(s^{(n)}, i)\}$ is bounded, since $\psi\geq 0$ and
if it were unbounded above then
$L_\psi(s^{(n)}, \eta^{(n)}) \geq \frac{\eta_i}{2}\psi(s^{(n)}, i)
\to \infty > L^*(\eta)$ eventually. Now suppose 
$C>0$ upper bounds $\{\psi_i(s^{(n)})\}$ for every $i\in I$. 
Since $\eta^{(n)}\to\eta$, There exists $N'$ such that $n>N'$ implies $\eta_i^{(n)} \geq 
\eta_i - \epsilon'/(MC)$ for every $i\in[M]$. Then, 
\begin{align*}
L_\psi(s^{(n)}, \eta^{(n)})-L_\psi(s^{(n)}, \eta) &= 
\sum_{i=1}^M (\eta^{(n)}_i - \eta_i)\psi(s^{(n)}, i) \\
&\geq \sum_{i\in I} (\eta^{(n)}_i - \eta_i)\psi(s^{(n)}, i)\\
& \geq M\left(\frac{-\epsilon'}{MC} C\right) = -\epsilon'.
\end{align*}
Therefore, for all $n>\max\{N, N'\}$, we have 
\begin{equation*}
|L_\psi(s^{(n)}, \eta^{(n)}) - L_\psi(s^{(n)}, \eta)| \leq \epsilon'.
\end{equation*}
Since $\epsilon'>0$ was arbitrary, this implies that 
$\{L_\psi(s^{(n)}, \eta)\}$ converges to the same limit as 
$\{L_\psi(s^{(n)}, \eta^{(n)})\}$. Thus, 
$L_\psi(s^{(n)}, \eta) \to L_\psi^*(\eta)$. We have thus reached 
the contradiction laid out earlier.
\end{proof}
\textbf{Proof of Theorem~\ref{thm:bd}}.
To prove Theorem~\ref{thm:bd}, we use the following two lemmas. The first establishes the openness of the set $\{s\in\R^M\mid \spy{k}{s}{\eta}\}$ for any $\eta\in\R^M$. The second says that a convex function with a unique minimizer has bounded sublevel sets.
\begin{lemma}
\label{lem:open}
$\sfP_k(\eta) := \{s\in\R^M\mid \spy{k}{s}{\eta}\}$ is open for any $\eta\in\R^M$, $k\in\Z^+$. 
\end{lemma}
\begin{proof}
Let $\eta\in\R^M$ and $s\in \sfP_k(\eta)$. Define 
\begin{align*}
    &\delta_1 = \min_{i\in[M]}\{s_i-s_{[k+1]}\mid s_i > 
    s_{[k+1]}\}\\
    &\delta_2 = \min_{i\in[M]}\{s_{[k]}-s_i\mid s_i < 
    s_{[k]}\}
\end{align*}
Take $\delta = \min\{\delta_1, \delta_2\}$, and notice $\delta >0$. Then, take $s'\in\R^M$ with
$|s'_i - s_i| < \delta/2$ for all $i\in[M]$. If $s_i > s_{[k+1]}$, then 
\begin{equation*}
  s_i' > s_i - \delta/2 > s_{[k+1]}+\delta/2> s'_{[k+1]},
\end{equation*} and similarly if $s_i < s_{[k]}$ then 
$s_i' < s'_{[k]}$. Therefore, $\spy{k}{s'}{\eta}$. This holds for every $s'$ in the neighborhood -- thus $\sfP_k(\eta)$ is open.
\end{proof}
\begin{lemma}
\label{lem:sublevel}
If $f:\R^M\to\R$ is convex and has a unique minimizer, the sublevel sets $\{x\in\R^M\mid f(x) \leq \alpha\}$ are bounded for every $\alpha\in\R$.
\end{lemma}
\begin{proof}
Suppose $x_0\in\R^M$ is the unique minimizer. We can assume $x_0 = 0$ by taking $f(x+x_0)$, which has the same sublevel sets just shifted by $x_0$, and a unique minimizer at $x=0$.

Then, $f(x) > f(0)$ for all $x\in\R^M$. Consider the set $B = \{x\in\R^M\mid \|x\|_2 = 1\}$. $B$ is compact. Therefore, the image of $B$ under $f$, $f(B)\subset\R$, is compact and has a minimum. Since $f(x)>f(0)$ for all $x\in B$, we have 
\begin{equation*}
\delta := \min(f(B)) - f(0) > 0.
\end{equation*}
Now, suppose $x\in\R^M$ such that $\|x\|_2 = D \geq 1$. Since $D\geq 1$, we have $0 < 1/D \leq  1$. Note $\|x/D\|_2 = 1$. Now we apply convexity: 
\begin{align*}
    f\left(\frac{x}{D}\right) \leq 
    \frac{1}{D}f(x) + \left(1-\frac{1}{D}\right)
    f(0).
\end{align*}
Rearranging, 
\begin{align*}
    f(x) &\geq Df\left(\frac{x}{D}\right) + (1-D)f(0)\\
    &= D(f(x/D)-f(0)) + f(0) \\
    &\geq D\delta + f(0).
\end{align*}
Thus, if $D\geq 1$, we have $\|x\|_2 \geq D$ implies 
$f(x) > D\delta/2 + f(0)$. The contrapositive is, $f(x) \leq D\delta/2 + f(0)$ implies $\|x\|_2 < D$ for $D\geq 1$. Therefore, for all $x\in\R^M$
\begin{equation*}
    f(x) \leq \alpha  \implies \|x\|_2 \leq \max\left\{\frac{2(\alpha-f(0))}{\delta}, 1\right\}. 
\end{equation*}
This says that the sublevel sets are bounded.
\end{proof}
Now we prove the theorem.
\begin{cthm}{\ref{thm:bd}}
Suppose $\phi:\R^M\to\R^M$ is strictly convex and differentiable. If $g:\R^M\to\R^M$ is inverse top-$k$ preserving, continuous, and $\Delta_M\subseteq \mathrm{range}(g)$, then $\psi:\R^M\times\cY\to\R$ defined by
\begin{equation*}
    \psi(s,y) = D_\phi(g(s), e_y)
\end{equation*}
is top-$k$ calibrated.
\end{cthm}
\begin{proof}
Let $\eta\in\Delta_M$. By Theorem 1 from \citet{Banerjee2005OnTO},
\begin{equation*}
    \argmin_{\bar\eta\in\R^M}\E_{Y\sim\eta}D_\phi(\bar\eta, Y)
    = \E[Y] = \eta.
\end{equation*}
We view the label $Y$ as an indicator vector in $\{0,1\}^M$ where the position of the one corresponds to the label. Therefore, 
\begin{align*}
\argmin_{s\in\R^M} L_\psi(s,\eta) 
&= \argmin_{s\in\R^M}\E_{Y\sim\eta}D_\phi(g(s), Y)\\
&= \{s\in\R^M\mid g(s) = \eta\},
\end{align*}
and since $\Delta_M\subseteq\mathrm{range}(g)$ the last set is nonempty. Let $s^*$ be such that $g(s^*)=\eta$.

Since $g$ is inverse top-$k$ preserving, 
$\spy{k}{s^*}{\eta}$. This holds for any $s^*$ in $O:=\{s\in\R^M\mid g(s) = \eta\}$. Given any $s$ for which
$\neg \spy{k}{s}{\eta}$, $s\not\in O$, and thus $g(s)\neq \eta$, $L_\psi(s,\eta) = \E_{Y\sim\eta} D_\phi(g(s), Y)
> \E_{Y\sim\eta} D_\phi(\eta, Y)$. Therefore,
\begin{equation*}
\inf_{s\in\R^M: \neg \spy{k}{s}{\eta}} L_{\psi}(s,\eta)
> \min_{s'\in\R^M}L_{\psi}(s', \eta).
\end{equation*}
To see this, first note $\E_{y\sim\eta} D_\phi(g, e_y)$ is convex in $g$ while attaining a unique minimum by \citet{Banerjee2005OnTO}. Therefore, by Lemma~\ref{lem:sublevel} the sublevel sets $\{g\mid \E_{y\sim\eta} D_\phi(g,e_y) \leq \alpha\}$ are bounded for any $\alpha\in\R$. 
Then 
\begin{align*}
\inf_{g\in\R^M: \neg \spy{k}{g}{\eta}} \E_{y\sim\eta} D_\phi(g,e_y) &= \min_{g\in\R^M: \neg \spy{k}{g}{\eta}} \E_{y\sim\eta} D_\phi(g,e_y)\\
&> \min_{s\in\R^M} L_\psi(s,\eta),
\end{align*}
as $\{g\in\R^M: \neg \spy{k}{g}{\eta}\}$ is closed by \ref{lem:open}, and for the infimum we only have to consider its intersection with some bounded closed (i.e. compact) set, due to the boundedness of the sublevel sets. Then since continuous functions map compact sets to compact sets, we can switch the infimum to a minimum.

Because $g$ is inverse top-$k$ preserving, $\spy{k}{s}{g(s)}$. Then, if $\spy{k}{g(s)}{\eta}$, 
 we see by transitivity of $\mathsf{P}_k$ that $\spy{k}{s}{\eta}$. Therefore, $\neg \spy{k}{s}{\eta} \implies \neg \spy{k}{g(s)}{\eta}$. So,   $A:=\{L_\psi(s,\eta)\mid 
\neg\spy{k}{s}{\eta}\}\subseteq \{\E_{y\sim\eta} D_\phi(g, \eta)
\mid \neg \spy{k}{g}{\eta}\}=:B,$ and  
\begin{equation*}
    \inf A\geq \min B > \min_{s\in\R^M} L_\psi(s, \eta).
\end{equation*}
Thus, $\psi$ is top-$k$ calibrated. 
\end{proof}


\begin{cthm}{\ref{psi1thm}}
Say a permutation $\pi:[M]\to[M]$ \emph{sorts} a vector $v\in\R^M$ if $v_{\pi_1}\geq v_{\pi_2}\geq\ldots\geq v_{\pi_M}$. Denote $S(v)$ as the set of permutations that sort $v$.

Let $\eta\in\Delta_M$, and suppose it has no zero entries. Then, for each of the following cases, the set of minimzers 
$\argmin_s L_{\psi_1} (s,\eta)$ is precisely described by the conditions on $s$ in the case.
\begin{align*}
&1.\quad 
\eta_{[k]} > \sum_{i=k+1}^M \eta_{[i]}:\; \exists c\in\R,\, \pi\in S(\eta)\\
&\quad 
\begin{aligned}
    &s_{\pi_{k+1}} = \ldots = s_{\pi_M} = c,\quad
    s_{\pi_k} = c+1,\\
    &\forall i\in \{1,\ldots, k-1\},\;s_{\pi_i} \in [c+1, \infty).
\end{aligned}\\
&2.\quad \eta_{[k]} < \sum_{i=k+1}^M \eta_{[i]}:\; \exists c\in\R,\, \pi\in S(\eta)\\
&\quad 
\begin{aligned}
    &s_{\pi_{k}} = \ldots = s_{\pi_M} = c,\\
    &\forall i\in \{1,\ldots, k-1\},\;s_{\pi_i} \in [c+1, \infty).
\end{aligned}\\
 &3.\quad \eta_{[k]}=\sum_{i=k+1}^M \eta_{[i]}:\; \exists c\in\R,\,\pi\in S(\eta) \\
 &\quad 
\begin{aligned}
    &s_{\pi_{k+1}} = \ldots = s_{\pi_M} = c,\quad
    s_{\pi_k} \in [c,c+1],\\
    &\forall i\in \{1,\ldots, k-1\},\;s_{\pi_i} \in [c+1, \infty).
\end{aligned}
\end{align*}
\end{cthm}
\begin{proof}
Suppose $\tau\in\Pi_M$ sorts $s$. Define $\delta := s_{\tau_k} - s_{\tau_{k+1}} = s_{[k]} - s_{[k+1]} \geq 0$. Since 
\begin{align*}
    &\max\{1+s_{\tau_{k+1}}-s_{\tau_k}, 0\}
    \geq \max\{1-\delta, 0\} \\
    &\max\{1+s_{\tau_{k}}-s_{\tau_i}, 0\}
    \geq 1+\delta,\,\,\forall i\in\{k+1,\ldots, M\},
\end{align*}
$L_\psi(s,\eta)$ is lower bounded as follows:
\begin{align}\nonumber
    L_{\psi}(s,\eta)&\geq \max\{1-\delta, 0\}\eta_{\tau_k} + (1+\delta)\sum_{i=k+1}^M \eta_{\tau_i} \\
    &\geq \max\{1-\delta, 0\}\eta_{[k]} + (1+\delta)\sum_{i=k+1}^M \eta_{[i]} =: F(\delta).
    \label{eq:ps1_proof}
\end{align}
In the following, we discuss when equality in $\eqref{eq:ps1_proof}$ is obtained in three cases. We may assume that $s_{\tau_{k+1}}$ is equal to an arbitrary $c\in\R$. Shifting each entry of $s$ by a constant does not change the loss value. Before we begin, we note common requirements, regardless of case. Since $\eta$ has no zero entries, the first line is an equality if and only if $s_{\tau_i} \geq s_{\tau_{k+1}} + 1 = c+1$ for all $i\in[k-1]$, and $s_{\tau_{k+1}} = s_{\tau_{k+2}} = \ldots = s_{\tau_M} = c$. And in any case where the second line is an equality, the sums on the right of both lines equal, which happens if and only if $\{\tau_{k+1},\ldots, \tau_M\} 
= \{\pi_{k+1},\ldots, \pi_M\}$ for some $\pi\in \Pi_M$ which sorts $\eta$.

Case 1:
If $\eta_{[k]} > \sum_{i=k+1}^M \eta_{[i]}$, $F(\delta)$ is minimized uniquely at $\delta = 1$ in the interval $[0,1]$; by our assumption that $\eta$ does not have 0 entries and $k < M$, $\delta > 1$ is suboptimal. Thus, $L_{\psi}^*(\eta) = 2\sum_{i=k+1}^M \eta_{[i]}$ (achieved by $s$ described below).

The equality is achieved if and only if the common requirements hold and $\delta =1$, giving $s_{\tau_k} = c+1$.

Case 2:
If $\eta_{[k]} < \sum_{i=k+1}^M \eta_{[i]}$, then $F(\delta)$ is minimized by $\delta$ = 0, and $L_{\psi}^*(\eta) = \sum_{i=k}^M \eta_{[i]}$. Therefore, the equality holds if and only if $s_{\tau_k} = s_{\tau_{k+1}} = c$ and $\tau_k$ = $\pi_k$ for some $\pi\in S_M$ which sorts $\eta$, along with the common requirements.

Case 3:
If $\eta_{[k]} = \sum_{i=k+1}^M \eta_{[i]}$, then $L_{\psi}^*(\eta) = \sum_{i=k}^M \eta_{[i]} = 2\sum_{i=k+1}^M \eta_{[i]}$. Thus $F(\delta)$ is minimized by $\delta\in[0,1]$.

If $\delta \in (0, 1)$, the inequality in $\eqref{eq:ps1_proof}$ requires
\begin{align*}
    &\sum_{i=k}^M \eta_{\tau_i} = \sum_{i=k}^M \eta_{[i]}
= 2\sum_{i=k+1}^M \eta_{\tau_i} = 2\sum_{i=k+1}^M \eta_{[i]}.
\end{align*}
Thus, the equality holds if and only if in addition to the common requirements, $s_{\tau_k} \in (c, c+1)$, and for some $\pi\in S_M$ which sorts $\eta$, 
$\pi_k = \tau_k$.

If $\delta = 1$ or $\delta = 0$, we have the same iff conditions for the equality as in case 1 and case 2.
\end{proof}

\begin{cprop}{\ref{prop:inconsistent}}
For any $\psi\in\{\psi_2, \psi_3, \psi_4\}$,
if $\sum_{m=k+1}^M \eta_{[m]} > \frac{k}{k+1}$, we have $0\in\argmin_s L_\psi(s, \eta)$, and thus $L^*_\psi(\eta) = \min_s L_\psi (s, \eta) = L_\psi(0, \eta) = 1$. 
\end{cprop}
\begin{proof}
We will show that $L^*_\psi(\eta)=1$. WLOG, we can assume that $\eta_1\geq\ldots\geq\eta_M$, $s_1\geq s_2\geq\ldots\geq s_M$, and $s_{k+1}=s_{k+2}=\ldots=s_M=0$. 

Suppose $s_i \geq 1$ for some $i\in[M]$. Then, for each $\psi\in\{\psi_2,\psi_3,\psi_4\}$, $\psi(s, i) \geq 1+\frac{1}{k}$ for all $i\in\{k+1,\ldots, M\}$, and so $L_\psi(s,\eta)\geq
\left(1+\frac{1}{k}\right)(\eta_{k+1}+\ldots+\eta_m) > \frac{k+1}{k}\cdot\frac{k}{k+1} = 1$. This implies that $s$ is suboptimal, since $L_\psi(0, \eta)=1$.

Thus, at optimum $0\leq s_i<1$ for every $i$, under which $\psi_2(s, i) = \psi_3(s,i) = \psi_4(s, i)$ for every $i$. This is because in this regime, $\max\{1+s_j - s_i, 0\}  = 
1+s_j-s_i$, and the $k$th highest value of $\bar{\mathbf{1}}(i) + s$ coincides with the $k$th highest value of $1+s$ excluding the $i$th index. 
Now for all $i\in[k]$, we have $s_i\in(0,1)$ and thus 
\begin{align*}
\frac{\partial L_\psi(s,\eta)}{
\partial s_i} &= \frac{1}{k}\sum_{m\in[M],m\neq i} \eta_m - \eta_i = \frac{1}{k}(1-\eta_i) - \eta_i\\ 
&> \frac{1}{k}\frac{k}{k+1} - \frac{1}{k+1}
= 0.
\end{align*}
The derivative is positive (and constant) in $(0,1)$, so the minimum value of $s_i$ is achieved at 0, for every $i$. Therefore, $L_\psi^*(\eta)=1$, achieved by a score vector of 0. This proves the desired statement.
\end{proof}
\begin{cprop}{\ref{prop:psi5}}
$\psi_5:\R^M\times\cY$ defined by 
$\psi_5(s, y) = \max\{1+s_{[k+1]}-s_y, 0\}$
is top-$k$ calibrated. 
\end{cprop}
\begin{proof}
Let $\eta\in\Delta_M$. For any $s\in\R^M$, we have 
\begin{equation*}
    L_{\psi_5}(s,\eta) = \sum_{i=1}^M\eta_i\psi_5(s,i)
    =\sum_{i=1}^M\eta_i\max\{1+s_{[k+1]}-s_i,0\}.
\end{equation*}
We may assume $\eta_1\geq \eta_2\geq \ldots \geq \eta_M$ WLOG. By inspection, setting $s_1=\ldots = s_k = 1$ and $s_{k+1}=\ldots=s_M=0$ gives $L_{\psi_5}(s,\eta) = 
\sum_{i=k+1}^M \eta_{[i]} =: C$. 

We will show that any $s\in\R^M$ such that $\neg\spy{k}{s}{\eta}$ has $L_\psi(s, \eta) - 
L_\psi^*(\eta) \geq L_\psi(s, \eta) - C \geq \delta$ for some constant $\delta > 0$, which implies top-$k$ calibration.

Suppose $\neg\spy{k}{s}{\eta}$. Define $\delta_1 = \min\{\eta_i-\eta_{[k+1]}\mid i\in[M], \eta_i > \eta_{[k+1]}\}$ and 
$\delta_2 = \min\{\eta_{[k]} - \eta_i\mid i\in[M], \eta_i < \eta_{[k]}\}$. If either set is empty, define its minimum to be $\infty$.
Furthermore, define the set $I:=\{i\in[M]\mid s_i\leq s_{[k+1]}\}$. Note by definition of $s_{[k+1]}$, $|I| \geq M-k$. We have 
$L_\psi(s, \eta) \geq \sum_{i\in I} \eta_i$.
There are two cases. 

If there exists $i\in[M]$ such that $\eta_i > \eta_{[k+1]}$ and $s_i \leq s_{[k+1]}$, then $i\in I$. But then $\sum_{j\in I} \eta_j \geq \sum_{j=k+1}^M \eta_{[j]} + \delta_1$. 

If there exists $i\in[M]$ such that $\eta_i < \eta_{[k]}$, but $s_i \geq s_{[k]}$, then consider if $s_i > s_{[k+1]}$. Then, $i\not\in I$. That is, $\eta_i$ does not appear in the sum $\sum_{j\in I} \eta_j$. Since $|I|\geq M-k$, $\eta_i$ must be replaced with a term $\eta_{i'} \geq \eta_{[k]}$. Thus, $\sum_{j\in I}\eta_j \geq \sum_{j=k+1}^M \eta_{[j]}+\delta_2$. If $s_i = s_{[k+1]}$, then since $s_i \geq s_{[k]} \geq s_{[k+1]}$, we have $s_i = s_{[k]}$. This implies $|I| > M-k$, and $\sum_{j\in I}\eta_j \geq \sum_{j=k}^M \eta_{[j]} \geq \sum_{j=k+1}^M \eta_{[j]} + \delta_2$. 

Thus, for any $s$ such that $\neg\spy{k}{s}{\eta}$, we have $L_\psi(s, \eta) \geq L_\psi^*(\eta) + \delta$ where $\delta = \min\{\delta_1, \delta_2\} > 0$. Therefore, 
\begin{equation*}
    \inf_{s:\neg\spy{k}{s}{\eta}}
    L_\psi(s, \eta) \geq \inf_s L_\psi(s, \eta) 
    + \delta > \inf_s L_\psi(s, \eta),
\end{equation*}
so $\psi=\psi_5$ is top-$k$ calibrated.
\end{proof}
\begin{cprop}{\ref{prop:fcon}}
Let $\cX=\R^d$ and $\cF = \{x\mapsto Wx: W\in\R^{M\times d}\}$. Then if we consider \textit{top}-$k$ \textit{separable} probability distributions over $\cX\times\cY$, i.e.
$L_{\err_k}^*(\cF) = 0 = L_{\err_k}^*$, then: 
\begin{enumerate}
\item If $k=1$, $\Ent$ is $\cF$-consistent.
    \item If $d\geq 3, M\geq 3$, and $k=2$, $\Ent$ is not $\cF$-consistent.
    \item $\psi_1$ and $\psi_5$ are $\cF$-consistent.
\end{enumerate}
\end{cprop}
\begin{proof}
\textbf{Proof of 2.}
 Let $\cX = \R^3$, $M=3$,
 and $k=2$. It does not matter if we increase dimensions or $M$. Let the dataset $S$ consist 
of the following 7 points, where $e_i$ denotes the standard basis element
with a 1 in the $i$th coordinate: $S = [2\times(e_1, 1), 2\times(e_2, 2), 
2\times(e_3,2), (-e_1, 1)]\subset \cX\times\cY$. The intuition is that having $e_1$ and
$-e_1$ both labeled 1 blatantly precludes linear separability.  \\
Note that $S$ is top-2 separable, since  $W_\text{sep}$ is a top-2 separator for $S$:
\begin{equation*}
W_\text{sep} = \begin{bmatrix}
2 & 0 & 0 \\
1 & 1 & 0 \\
3 & 0 & 1
\end{bmatrix}.
\end{equation*}
The following score vectors are returned for each input:
\begin{align*}
&W_\text{sep}e_1 = \begin{bmatrix} 2\\1\\3\end{bmatrix}\quad
W_\text{sep}e_2 = \begin{bmatrix} 0\\1\\0\end{bmatrix}\\
&W_\text{sep}e_3 = \begin{bmatrix} 0\\0\\1\end{bmatrix}\quad
W_\text{sep}(-e_1) = \begin{bmatrix} -2\\-1\\-3\end{bmatrix}.
\end{align*}
$W_\text{sep}e_2$ and $W_\text{sep}e_3$ respectively have
their second and third entries
 as their strictly greatest entries. Since they are 
respectively labeled 2 and 3, 
they are classified correctly. $W_\text{sep}e_1$ and $W_\text{sep}(-e_1)$ 
both have their first entry strictly greater than their third entry.
This means they are classified correctly by a top-2 classifier, as
their label is 1: $(W_\text{sep} e)_1 > (W_\text{sep} e)_{[3]}$
for $e\in\{e_1, -e_1\}$.  \\
Now we show the solution returned by cross entropy minimization is not 
a top-$k$ separator. $\Ent(W,S)$ denotes the cross entropy loss incurred by $W$ on the probability distribution defined by the dataset $S$, times the number of samples.
\begin{align*}
\Ent(W, S) &=  \sum_{(x,y)\in S} \log\left(\sum_{m=1}^3 e^{(Wx)_m
- (Wx)_y}\right)  \\
&=2\sum_{i=1}^3 \log\left(\sum_{m=1}^3 e^{(We_i)_m - (We_i)_i}\right)\\
&\qquad+ \log\left(\sum_{m=1}^3 e^{(We_1)_1 - (We_1)_m}\right)\\
&=2\sum_{i=1}^3 \log\left(\sum_{m=1}^3 e^{W_{mi} - W_{ii}}\right)\\
&\qquad+ \log\left(\sum_{m=1}^3 e^{W_{11} - W_{m1}}\right).
\end{align*}
For each different $i$, the entries of $W$ appearing in the $i$th term 
in the sum correspond to different columns of $W$ -- entries 
appearing in different terms are independent of each other. For 
$i\neq 1$, we see that $\log\left(\sum_{m=1}^3 e^{W_{mi}-W_{ii}}\right)
= \log\left(1+\sum_{m\neq i} e^{W_{mi}-W_{ii}}\right)$ can be
taken to 0 by taking $W_{mi}-W_{ii} \to -\infty$ for 
each $m\neq i$. We cannot do the same for $i=1$ because of the 
appearance of both $(e_1,1)$ and $(-e_1,1)$. But at this point, we 
have gotten rid of terms with $i\neq 1$ and determined that the 
minimizer of $\Ent$ looks like the following:
\begin{equation*}
W_{CE} = \begin{bmatrix}
? & W_{22}-\infty & W_{33}-\infty \\
? & W_{22} & W_{33}-\infty\\
? & W_{22}-\infty & W_{33}
\end{bmatrix}.
\end{equation*}
The remainder of the loss function is 
\begin{align*}
\Ent(W,S) =& 2\log\left(1+ e^{W_{21}-W_{11}} + e^{W_{31}-W_{11}}\right)\\
&+ \log\left(1+e^{W_{11}-W_{21}}+e^{W_{11}-W_{31}}\right).
\end{align*}
Denote $x_1 = W_{21}-W_{11}$ and $x_2 = W_{31}-W_{11}$, so we may 
write the loss as 
\begin{equation*}
\Ent(W,S) = 2\log\left(1+e^{x_1}+e^{x_2}\right) + \log\left(1+
e^{-x_1} + e^{-x_2}\right).
\end{equation*}
We have 
\begin{align*}
&\pd{\Ent}{x_1} = \frac{2e^{x_1}}{1+e^{x_1}+e^{x_2}} - \frac{e^{-x_1}}{
1+e^{-x_1}+e^{-x_2}},\\
&\pd{\Ent}{x_2} = \frac{2e^{x_2}}{1+e^{x_1}+e^{x_2}} - \frac{e^{-x_2}}{
1+e^{-x_1}+e^{-x_2}}.
\end{align*}
By the convexity of $\log(1+e^{x_1}+e^{x_2})$, we may minimize the 
function by setting the derivatives equal to 0. Note that if 
$x_1\neq x_2$, this is not achievable -- suppose it were the case 
that $\pd{\Ent}{x_1} = 0$. If $x_2 > x_1$, then $e^{x_2} > e^{x_1}$ and 
$e^{-x_2} < e^{-x_1}$, so $\pd{L}{x_2} > \pd{L}{x_1} = 0$. A similar 
argument holds if $x_2 < x_1$. Therefore, we may assume $x_1=x_2$.
Then we simply need 
\begin{gather*}
\frac{2e^x}{1+2e^x} - \frac{e^{-x}}{1+2e^{-x}} = 0 
\iff 2e^x + 4 - e^{-x} - 2 = 0\\
\iff 2e^x - e^{-x} = -2.
\end{gather*}
If $x\geq 0$, then clearly $2e^x - e^{-x} > 0$. Thus, $x < 0$
(we can solve a quadratic, or note that there exists $x$ where 
$2e^x - e^{-x} = -2$ because the LHS goes to $-\infty$ as $x\to-\infty$).
Therefore, at minimum $W_{21}-W_{11} = W_{31}-W_{11} = x$
for some $x<0$, so the 
cross entropy minimizer is the following: 
\begin{equation*}
W_{CE} = 
\begin{bmatrix}
W_{11} & W_{22}-\infty & W_{33}-\infty \\
W_{11}+x & W_{22} & W_{33}-\infty\\
W_{11}+x & W_{22}-\infty & W_{33}
\end{bmatrix}.
\end{equation*}
This is not a top-$2$ separator because $W_{CE}(-e_1) = \begin{bmatrix}
-W_{11},& -W_{11}-x, & -W_{11}-x,\end{bmatrix}^\top$, whose first entry 
is strictly the lowest entry since $x<0$. Thus, $-e_1$ is not 
classified as its label, 1.\\
\textbf{Proof of 3.}
Recall $\psi_1, \psi_5$: 
\begin{align*}
&\psi_1(s,y) = \max\{1+(s_{\setminus y})_{[k]} - s_y, 0\},\\
&\psi_5(s,y) = \max\{1+s_{[k+1]} - s_y, 0\}.
\end{align*}
We will show these losses are linearly top-$k$ consistent. Suppose 
$S=((x_1,y_1), \ldots, (x_n,y_n))$ is top-$k$ separable, that is, 
$\exists\,W\in\R^{M\times d}$ such that $\forall\,i\in[n],\;
(Wx_i)_{y_i} > (Wx_i)_{[k+1]}$. In other words, there is a 
$\delta >0$ such that for every 
$i\in [n],\; (Wx_i)_{y_i}-(Wx_i)_{[k+1]} \geq \delta$. Then, for 
$C\geq\frac{1}{\delta}, \; (CWx_i)_{y_i} - (CWx_i)_{[k+1]} \geq C\delta
\geq1$ for every $i\in[n]$. \\
Now let $i\in[n]$ and denote $s=CWx_i$. Since $s_{y_i} > s_{[k+1]}$,
we have $s_{[k+1]} = (s_{\setminus y})_{[k]}$. Thus, 
\begin{equation*}
\psi_1(s,y_i) = \psi_5(s,y_i) = \max\{1+s_{[k+1]}-s_{y_i},0\} = 0.
\end{equation*} 
Therefore, $CW$ achieves 0 loss on the dataset for both $\psi_1$ and 
$\psi_5$. This means their minimizers (over linear functions)
 achieve 0 loss. If 0 loss 
is achieved, it is clear that the resulting classifiers achieve 
0 top-$k$ error, since these losses upper bound the top-$k$ error. 
Therefore, their minimizers are top-$k$ separators. \\
We have shown that if a dataset is linearly top-$k$ separable, then the 
minimizers of $\psi_1$ and $\psi_5$ are top-$k$ linear separators for the 
dataset. This proves that $\psi_1$ and $\psi_5$ are linearly top-$k$ 
\end{proof}
\section{Discussion of general hinge-like losses}
Recall that the hinge loss for binary classification is defined by $\phi(x)=\max\{1-x, 0\}$. There are several extensions of the binary hinge loss to the setting of multiclass classification (often with multiclass error i.e. top-1 loss). We list them here because they serve as inspiration for designing hinge-like top-$k$ losses, and the analysis of their consistency in the literature also informs the analysis of the top-$k$ case. 

The method of \citet{Crammer2001OnTA} uses as its loss function $\psi:\R^M\times \cY\to \R$ where
\begin{equation}
\label{crammer}
    \psi(s,y)=\max\{1+(s_{\setminus y})_{[1]} - s_y, 0\}
    = \phi(s_y - \max_{y'\neq y} s_{y'}).
\end{equation}
When $y\in\cY$ appears in a subscript it refers to the label as an index in $\{1,\ldots, M\}$. Furthermore, the notation $s_{\setminus y} = (s_1, \ldots, s_{y-1}, s_{y+1},\ldots, s_M)\in \R^{M-1}$ denotes the vector $s$ with the $y$th entry removed.

The method of \citet{weston1999} solves a multiclass SVM problem for which the corresponding loss function is 
\begin{equation*}
    \psi(s,y) = \sum_{y'\neq y}\phi(s_y-s_{y'}),
\end{equation*}
where $\phi$ is still the binary hinge loss. 
Furthermore, the one vs. all method \citet{rifkin2004defense} solves $M$ binary classification problems using the hinge loss for each class, using the instances of the class as positive examples and the rest of the instances as negative examples. The $M$ scores returned by the $M$ resulting classifiers are compiled into an $M$ length vector, and the method proceeds like all the above methods by taking the argmax of the vector. 
Similarly, the method of \citet{lee2004} minimizes the expectation of the loss function 
\begin{equation*}
    \psi(s,y) = \sum_{y'\neq y}\phi(-s_{y'})
\end{equation*}
under the constraint that $\sum_{m=1}^M s_m = 0$. 
Interestingly, \citet{Zhang2004StatisticalAO} showed the first three \citet{Crammer2001OnTA, weston1999, rifkin2004defense} to be inconsistent, i.e. not top-1 calibrated, and the constrained \citet{lee2004} to be consistent. These results were also found by \citet{Tewari2005OnTC}.

\begin{table*}
    \caption{Results of the first synthetic data experiment: Predicted score vector $s=f(0)$ with the zero vector as input.}
    \vspace{0.2cm}
    \label{tab:synexp1.2}
    \begin{tabular}{ccccccccc}
    \toprule
    & $s_1$ & $s_2$
    & $s_3$ & $s_4$
    & $s_5$ & $s_6$
    & $s_7$ & $s_8$\\\midrule
$\psi_1$ &  0.87793601 & -0.12823531 & -0.12382337 &-0.12676451 &   -0.12382337& -0.12235278
& -0.12529394 & -0.12764691 \\
$\psi_2$ &  0.00176411 & 0.00044059  & -0.00058873  &-0.00176518 & -0.00220636 & 0.0002936
 & 0.00073477 & 0.00132302 \\
$\psi_3$ & 0.00117588  & 0.00191117&  0.00102892 &-0.0010299 & -0.0020593 & -0.00029462
 & 0.00073478& -0.00147108 \\ 
$\psi_4 $ & 0.00073472&  0.00161706 & 0.00029361 &-0.00264753 & 0.00117595 & 0.00088184
& -0.00191224& -0.00014757 \\
$\psi_5 $ & 0.75734961&  0.75734961& -0.25529474 &-0.24823636& -0.2523534&  -0.24823636
& -0.25529483& -0.25529486 \\\bottomrule
    \end{tabular}
\end{table*}

\end{document}